\documentclass{article}

\newif\ifarxiv
\arxivtrue

\ifarxiv
   \usepackage[accepted]{icml2020}
\else
   \usepackage{icml2020}
\fi

\usepackage{times}

\usepackage[utf8]{inputenc}   % allow utf-8 input
\usepackage[T1]{fontenc}      % use 8-bit T1 fonts
\usepackage[
   colorlinks,
   citecolor=teal]{hyperref}  % hyperlinks
\usepackage{url}              % simple URL typesetting
\usepackage{booktabs}         % professional-quality tables
\usepackage{amsfonts}         % blackboard math symbols
\usepackage{nicefrac}         % compact symbols for 1/2, etc.
\usepackage{microtype}        % microtypography
\usepackage{amsmath}
\usepackage{pifont}
\usepackage{amsthm}
\usepackage{algorithm,algorithmicx,algpseudocode}
\usepackage{mathtools}
\usepackage{xcolor}
\usepackage{comment}
\usepackage{caption}
\usepackage{subcaption}
\usepackage{wrapfig}
\usepackage{enumitem}
\usepackage{arydshln}

\newtheorem{thm}{Theorem}

\newcommand{\ie}{\textit{i.e.,} }
\newcommand{\eg}{\textit{e.g.,} }

\newcommand{\norm}[1]{\left\lVert{#1}\right\rVert}

\newcommand{\R}{\mathbb{R}}

\newcommand{\thmref}[1]{Theorem~\ref{#1}}

\renewcommand{\eqref}[1]{(\ref{#1})}

\newcommand{\currw}{w_t}
\newcommand{\nextw}{w_{t+1}}

\newcommand{\wi}{w_i}
\newcommand{\currwi}{w_{t,i}}
\newcommand{\w}{w}

\newcommand{\fs}{f_s}
\newcommand{\fst}{f_{s_t}}

\newcommand{\currg}{g_t}

\newcommand{\prevm}{m_{t-1}}
\newcommand{\currm}{m_t}

\newcommand{\prevv}{v_{t-1}}
\newcommand{\currv}{v_t}

\newcommand{\preve}{\eta_{t-1}}
\newcommand{\curre}{\eta_t}
\newcommand{\currei}{\eta_{t,i}}

\newcommand{\curra}{\alpha_t}

\newcommand{\normed}[1]{\left\lVert {#1} \right\rVert}

\newcommand{\btwo}{\beta_2}
\newcommand{\bone}{\beta_1}
\newcommand{\btwot}{\beta_{2,t}}
\newcommand{\bonet}{\beta_{1,t}}

\newcommand{\fdist}{D}
\newcommand{\smooth}{M}
\newcommand{\lowinf}{L}
\newcommand{\lowinft}{L_{t}}
\newcommand{\highinf}{H}
\newcommand{\highinft}{H_{t}}
\newcommand{\gradb}{G_\infty}
\newcommand{\gradbtwo}{G_2}

\newcommand{\supp}{\mathcal S}

\newcommand{\expec}[2]{\mathbb E_{#1} \left[ {#2} \right]}

\newcommand{\algrand}{\substack{S \sim \dist^T \\ t \sim \mathcal P(t|S)}}

\newcommand{\dist}{\mathcal D}

\icmltitlerunning{Domain-independent Dominance of Adaptive Methods}

\begin{document}
\twocolumn[
\icmltitle{Domain-independent Dominance of Adaptive Methods}

\icmlsetsymbol{equal}{*}

\begin{icmlauthorlist}
\icmlauthor{Pedro Savarese}{tti}
\icmlauthor{David McAllester}{tti}
\icmlauthor{Sudarshan Babu}{tti}
\icmlauthor{Michael Maire}{uc}
\end{icmlauthorlist}

\icmlaffiliation{tti}{TTI-Chicago}
\icmlaffiliation{uc}{University of Chicago}

\icmlcorrespondingauthor{Pedro Savarese}{savarese@ttic.edu}

\icmlkeywords{Machine Learning, ICML}

\vskip 0.3in
]

\printAffiliationsAndNotice{}
\begin{abstract}
From a simplified analysis of adaptive methods, we derive AvaGrad, a new optimizer which outperforms SGD on vision tasks when its adaptability is properly tuned. We observe that the power of our method is partially explained by a decoupling of learning rate and adaptability, greatly simplifying hyperparameter search. In light of this observation, we demonstrate that, against conventional wisdom, Adam can also outperform SGD on vision tasks, as long as the coupling between its learning rate and adaptability is taken into account. In practice, AvaGrad matches the best results, as measured by generalization accuracy, delivered by any existing optimizer (SGD or adaptive) across image classification (CIFAR, ImageNet) and character-level language modelling (Penn Treebank) tasks.
\end{abstract}
\section{Introduction}
\label{sec:intro}

Deep network architectures are becoming increasingly complex, often containing
parameters that can be grouped according to multiple functionalities, such as
gating, attention, convolution, and generation.  Such parameter groups should
arguably be treated differently during training, as their gradient statistics
might be highly distinct.  Adaptive gradient methods designate parameter-wise
learning rates based on gradient histories, treating such parameters groups
differently and, in principle, promise to be better suited for training complex
neural network architectures.

Nonetheless, advances in neural architectures have not been matched by progress
in adaptive gradient descent algorithms.  SGD is still prevalent, in
spite of the development of seemingly more sophisticated adaptive alternatives,
such as RMSProp \citep{rmsprop} and Adam \citep{adam}.  Such adaptive methods
have been observed to yield poor generalization compared to SGD in
classification tasks \citep{marginal}, and hence have been mostly adopted for
training complex models \citep{transformers, wgan}.  For relatively simple
architectures, such as ResNets \citep{resnet1} and DenseNets \citep{densenet},
SGD is still the dominant choice.

At a theoretical level, concerns have also emerged about the current crop of
adaptive methods.  Recently, \citet{amsgrad} has identified cases, even in the
stochastic convex setting, where Adam \citep{adam} fails to converge.
Modifications to Adam that provide convergence guarantees have been formulated,
but have shortcomings.  AMSGrad \citep{amsgrad} requires non-increasing
learning rates, while AdamNC \citep{amsgrad} and AdaBound \citep{adabound}
require that adaptivity be gradually eliminated during training.  Moreover,
while most of the recently proposed variants do not provide formal guarantees
for non-convex problems, the few current convergence rate analyses in the literature \citep{yogi,adamlike} do not match SGD's.  Section~\ref{sec:related} fully
details the convergence rates of the most popular Adam variants, along with
their shortcomings.

Our contribution is marked improvements to adaptive optimizers, from both
theoretical and practical perspectives.  At the theoretical level, we focus on
convergence guarantees, deriving new algorithms:

\vspace{-5.0pt}
\begin{itemize}[leftmargin=0.2in]
   \item{
      \textbf{Delayed Adam.} Inspired by \citet{yogi}'s analysis of Adam,
      Section~\ref{sec:dadam} proposes a simple modification for adaptive
      gradient methods which yields a provable convergence rate of
      $O(1 / \sqrt T)$ in the stochastic non-convex setting -- the same as SGD.
      Our modification can be implemented by \emph{swapping two lines of code}
      and preserves adaptivity without incurring extra memory costs.  To
      illustrate these results, we present a non-convex problem where Adam
      fails to converge to a stationary point, while Delayed Adam -- Adam with
      our proposed modification -- provably converges with a rate of
      $O(1 / \sqrt T)$.
   }
   \item{
      \textbf{AvaGrad.}
      Inspecting the convergence rate of Delayed Adam, we show that it would
      improve with an adaptive global learning rate, which self-regulates based
      on global statistics of the gradient second moments.  Following this
      insight, Section~\ref{sec:ava} proposes a new adaptive method, AvaGrad,
      whose hyperparameters decouple learning rate and adaptability.
   }
\end{itemize}

\pagebreak

Through extensive experiments, Section~\ref{sec:exp} demonstrates that AvaGrad
is not merely a theoretical exercise.  AvaGrad performs as well as both SGD and
Adam in their respectively favored usage scenarios.  Along this experimental
journey, we happen to disprove some conventional wisdom, finding adaptive
optimizers, including Adam, to be superior to SGD for training CNNs.  The
caveat is that, excepting AvaGrad, these methods need extensive grid search to outperform SGD, often requiring unconventional hyperparameter values to even yield competitive performance.

AvaGrad is a uniquely attractive adaptive optimizer, as it decouples the learning rate and its adaptability parameter, making hyperparameter search significantly faster. In particular, given a computational budget similar to SGD's, AvaGrad yields near best results over a wide range of tasks. 

\section{Preliminaries}
\label{sec:prelim}

\subsection{Notation}

For vectors $a = [a_1, a_2, \dots], b = [b_1, b_2, \dots] \in \R^d$, we use the
following notation:
   $\frac1{a}$ for element-wise division
      ($\frac1{a} = [\frac1{a_1}, \frac1{a_2}, \dots]$),
   $\sqrt a$ for element-wise square root
      ($\sqrt a = [\sqrt{a_1}, \sqrt{a_2}, \dots]$),
   $a + b$ for element-wise addition
      ($a+b = [a_1+b_1, a_2+b_2, \dots]$),
   $a \odot b$ for element-wise multiplication
      ($a \odot b = [a_1 b_1, a_2 b_2, \dots]$).
Moreover, $\normed{a}$ is used to denote the $\ell_2$-norm: other norms will be
specified whenever used (\eg $\normed{a}_\infty$).

For subscripts and vector indexing, we adopt the following convention: the
subscript $t$ is used to denote an object related to the $t$-th iteration of an
algorithm (\eg $\currw \in \R^d$ denotes the iterate at time step $t$); the
subscript $i$ is used for indexing: $\wi \in \R$ denotes the $i$-th coordinate
of $\w \in \R^d$.  When used together, $t$ precedes $i$: $\currwi \in \R$
denotes the $i$-th coordinate of $\currw \in \R^d$.

\subsection{Stochastic Non-Convex Optimization}

In the stochastic non-convex setting, we are concerned with the optimization
problem:
\begin{equation}
   \min_{\w \in \R^d} f(w) = \expec{s \sim \dist}{\fs(\w)}
\end{equation}
where $\dist$ is a probability distribution over a set $\supp$ of
``data points''.  We also assume that $f$ is $\smooth$-smooth in $w$, as is
typically done in non-convex optimization:
\begin{equation}
   \forall~w,\w'\;
      f (\w') \leq
      f (\w) +
      \langle \nabla f (\w), \w' - \w \rangle +
      \frac{\smooth}{2} \normed{\w - \w'}^2
\end{equation}

Methods for stochastic non-convex optimization are evaluated in terms of number
of iterations or gradient evaluations required to achieve small loss gradients.
This differs from the stochastic convex setting where convergence is measured
w.r.t.~suboptimality $f(\w) - \min_{\w \in \R^d} f(\w)$.  We assume that the
algorithm takes a sequence of data points $S = (s_1, \ldots, s_T)$ from which
it deterministically computes a sequence of parameter settings
$w_1, \ldots, w_T$ together with a distribution ${\cal P}$ over
$\{1,\ldots,T\}$.  We say an algorithm has a convergence rate of $O(g(T))$ if
$\expec{\algrand}{\normed{\nabla f(w_t)}^2}
   \leq O(g(T))$ where, as defined above, $f(w) = \expec{s\sim \dist}{\fs(w)}$.

We also assume that the functions $\fs$ have bounded gradients: there exists
some $\gradb$ such that
   $\normed{\nabla \fs (\w)}_\infty \leq \gradb$
for all $s \in \supp$ and $\w \in \R^d$.  Throughout the paper, we also let
$\gradbtwo$ denote an upper bound on $\norm{\nabla \fs (\w)}$.

\section{Related Work}
\label{sec:related}

Here we present a brief overview of optimization methods commonly used for
training neural networks, along with their convergence rate guarantees for
stochastic smooth non-convex problems.  We consider methods which, at each
iteration $t$, receive or compute a gradient estimate:
\begin{equation}
   \currg \coloneqq \nabla \fst(\currw), \quad s_t \sim \dist
\end{equation}
and perform an update of the form:
\begin{equation}
    \nextw = \currw - \curra \cdot \curre \odot \currm
\end{equation}
where $\curra \in \R$ is the \textbf{global learning rate},
$\curre \in \R^d$ are the \textbf{parameter-wise learning rates}, and
$\currm \in \R^d$ is the update direction, typically defined as:
\begin{equation}
   \currm = \bonet \prevm + (1 - \bonet) \currg
   \quad \textrm{and} \quad m_0 = 0.
\end{equation}
Non-momentum methods such as SGD, AdaGrad, and RMSProp \citep{rmsprop, adagrad}
have $\currm = \currg$ (\ie $\bonet = 0$), while momentum SGD and Adam
\citep{adam} have $\bonet \in (0,1)$.  Note that while $\curra$ can always be
absorbed into $\curre$, representing the update in this form will be convenient
throughout the paper.

SGD uses the same learning rate for all parameters, \ie $\curre = \vec 1$.
Although SGD is simple and offers no adaptation, it has a convergence rate of
$O(1 / \sqrt T)$ with either constant, increasing, or decreasing learning rates
\citep{nonconvex}, and is widely used when training deep networks, especially
CNNs \citep{resnet1, densenet}.  At the heart of its convergence proof is the
fact that
   $\expec{s_t}{\curra \cdot \curre \odot \currg} =
      \curra \cdot \nabla f(\currw)$.

Popular adaptive methods such as RMSProp \citep{rmsprop}, AdaGrad
\citep{adagrad}, and Adam \citep{adam} have
   $\curre = \frac1{\sqrt{\currv} + \epsilon}$,
where $\currv \in \R^d$ is given by:
\begin{equation}
   \currv = \btwot \prevv + (1-\btwot) \currg^2
   \quad \textrm{and} \quad v_0 = 0.
\end{equation}
As $\currv$ is an estimate of the second moments of the gradients, the
optimizer designates smaller learning rates for parameters with larger
uncertainty in their stochastic gradients.  However, in this setting $\curre$
and $s_t$ are no longer independent, hence
   $\expec{s_t}{\curra \cdot \curre \odot \currg} \neq
      \curra \cdot \expec{s_t}{\curre} \odot \nabla f(\currw)$.
This ``bias'' can cause RMSProp and Adam to present convergence issues, even in
the stochastic convex setting \citep{amsgrad}.

Recently, \citet{yogi} showed that, with a constant learning rate, RMSProp and Adam have a convergence rate of
$O(\sigma^2 + 1/T)$, where
   $\sigma^2 = \sup_{\w \in \R^d}
      \expec{s \sim \dist}{\normed{\nabla \fs(\w) - \nabla f(\w)}^2}$,
hence their result does not generally guarantee convergence. 

\citet{adamlike} showed that
AdaGrad and AMSGrad enjoy a convergence rate of $O(\log T / \sqrt T)$ when a decaying learning rate is used. Note
that both methods constrain $\curre$ in some form, the former with
   $\btwot = 1 - 1/t$ (adaptability diminishes with $t$),
and the latter explicitly enforces
   $\currv \geq v_j$ for all $j < t$ ($\curre$ is point-wise non-increasing).
In both cases, the method is less adaptive than Adam, and the rates above are worse than SGD's $O(1 / \sqrt T)$.
\section{SGD-like Convergence without Constrained Rates}
\label{sec:dadam}

We first take a step back to note the following: to show that Adam might not
converge in the stochastic convex setting, \citet{amsgrad} provide a
stochastic linear problem where Adam fails to converge w.r.t. suboptimality.
Since non-convex optimization is evaluated w.r.t. norm of the gradients, a
different instance is required to characterize Adam's behavior in this setting.

The following result shows that even for a quadratic problem, Adam indeed does
not converge to a stationary point:

\begin{thm}
   For any $\epsilon\geq0$ and constant $\btwot = \btwo \in [0,1)$, there is a
   stochastic convex optimization problem for which Adam does not converge to
   a stationary point.
   \label{thm:adamdiv}
\end{thm}
\begin{proof}
   The full proof is given in Appendix \ref{sec:proof1} \footnote{See suplementary material for Appendices}.  The argument follows
   closely from \citet{amsgrad}, where we explicitly present a stochastic
   optimization problem:
   \begin{equation}
   \begin{split}
      & \min_{\w \in [0,1]} f(\w) \coloneqq \expec{s \sim \dist}{\fs(\w)} \\
      & \fs(\w) =
         \begin{cases}
            C \frac{\w^2}2, \quad
               \text{with probability } \quad
               p \coloneqq \frac{1+\delta}{C+1} \\
            -\w, \quad \text{otherwise}
         \end{cases}
   \end{split}
   \end{equation}
   We show that, for large enough $C$ (as a function of
   $\delta, \epsilon, \btwo$), Adam will move towards
   $\w = 1$ where $\nabla f(1) = \delta$, and that the constraint
   $w \in [0,1]$ does not make $w=1$ a stationary point.
\end{proof}

This result, like the one in \citet{amsgrad}, relies on the fact that $\curre$
and $s_t$ are correlated: upon a draw of the rare sample $C \frac{\w^2}2$, the
learning rate $\curre$ decreases significantly and Adam takes a small step in
the correct direction.  On the other hand, a sequence of common samples
increases $\curre$ and Adam moves faster towards $w=1$.

Instead of enforcing $\curre$ to be point-wise non-increasing in $t$
\citep{amsgrad}, which forces the optimizer to take small steps even for a long
sequence of common samples, we propose to simply have $\curre$ be independent
of $s_t$.  As an extra motivation for this approach, note that successful
proof strategies \citep{yogi} to analyzing adaptive methods include the
following step:
\begin{equation}
\begin{split}
   \expec{s_t}{\curre \odot \currg}
      &= \expec{s_t}{ \left(\preve + \curre - \preve \right) \odot \currg } \\
      &= \preve \odot \nabla f(\currw) +
         \expec{s_t}{ \left(\curre - \preve \right) \odot \currg}
\end{split}
\label{eq:proofstep}
\end{equation}
where bounding $\expec{s_t}{ \left(\curre - \preve \right) \odot \currg}$, seen
as a form of bias, is a key part of recent convergence analyses.  Replacing
$\curre$ by $\preve$ in the update equation of Adam removes this bias and can
be implemented by simply swapping lines of code (updating $\eta$ \textit{after}
$\w$), yielding a simple convergence analysis without hindering the adaptability of the method in any way.

Algorithm~\ref{alg:dadam} provides pseudo-code when applying this modification to Adam, yielding Delayed Adam.  The following Theorem shows that this
modification is enough to guarantee a SGD-like convergence rate of
$O(1 / \sqrt T)$ in the stochastic non-convex setting for general adaptive
gradient methods.

\begin{thm}
   \label{thm:dadamconv}
   Consider any optimization method which updates parameters as follows:
   \begin{equation}
      \nextw = \currw - \curra \cdot \curre \odot \currg
   \end{equation}
   where $\currg \coloneqq \nabla \fst(\currw)$, $s_t \sim \dist$, and
   $\curra, \curre$ are independent of $s_t$. Assume that $f(\w_1) - f(\w^\star) \leq \fdist$,
   $f(w) = \expec{s \sim \dist}{\fs(\w)}$ is $\smooth$-smooth, and
   $\normed{\nabla \fs (\w)}_\infty \leq \gradb$ for all
   $s \in \supp, w \in \R^d$.
   Moreover, let $Z = \sum_{t=1}^T \alpha_t \min_i \currei$.

   For $\curra = \gamma_t \sqrt{\frac{2 \fdist}{T \smooth \gradb^2}}$,
   if $p(Z|s_t) = p(Z)$ for all $s_t \in \supp$, then:
   \begin{equation}
   \begin{split}
      & \expec{\algrand}{\normed{\nabla f(\currw)}^2} \\
         & \quad\quad\quad \leq \sqrt{\frac{\smooth \fdist \gradb^2}{2T}} \cdot
            \expec{S \sim \dist^T}{\frac{\sum_{t=1}^T 1 +
               \gamma_t^2 \norm{\curre}^2}{\sum_{t=1}^T \gamma_t \min_i \currei}}
   \end{split}
   \label{eq:convrate}
   \end{equation}
   where $\cal P$ assigns probabilities
      $p(t) \propto \curra \cdot \min_i \currei$.
\end{thm}
\begin{proof}
   The full proof is given in Appendix \ref{sec:proof2}, along with analysis
   for the case with momentum $\bonet \in (0,1)$ in
   Appendix~\ref{sec:proof2-momentum}, and in particular
   $\bonet = \bone / \sqrt t$, which yields a similar rate.
\end{proof}

The convergence rate depends on $\norm{\curre}$ and $\min_i \currei$, which are
random variables for Adam-like algorithms.  However, if there are constants
$\highinf$ and $\lowinf$ such that
\begin{equation}
0 < \lowinf \leq \currei \leq \highinf < \infty
\end{equation}
for all $i$ and $t$, then a
rate of $O(1 / \sqrt T)$ is guaranteed.

This is the case for Delayed Adam,
where $1/(\gradbtwo + \epsilon) \leq \currei \leq 1 / \epsilon$ for all $t$ and
$i$.  \thmref{thm:dadamconv} also requires that $\curra$ and $\curre$ are
independent of $s_t$, which can be assured to hold by applying a ``delay'' to
their respective computations, if necessary (\ie replacing $\curre$ by
$\preve$, as in Delayed Adam).

Additionally, the assumption that $p(Z|s_t) = p(Z)$, meaning that a single
sample should not affect the distribution of
$Z = \sum_{t=1}^T \alpha_t \min_i \currei$, is required since $\mathcal P$ is
conditioned on the samples $S$ (unlike in standard analysis, where
$Z = \sum_{t=1}^T \curra$ and $\curra$ is deterministic), and is expected to
hold as $T \to \infty$.

Practitioners typically use the last iterate $w_T$ or
perform early-stopping: in this case, whether the assumption holds or not does
not affect the behavior of the algorithm.  Nonetheless, we also show in
Appendix~\ref{sec:proof2-unconditional} a similar rate that does not require
this assumption to hold, which also yields a $O(1 / \sqrt T)$ convergence rate
taken that the parameter-wise learning rates are bounded from above and below.

\section{AvaGrad: An Adaptive Method with Adaptive Variance}
\label{sec:ava}
\begin{algorithm}[t]
   \caption{\textsc{Delayed Adam}}
   \label{alg:dadam}
   \textbf{Input:}
      $\w_1 \in \R^d$, $\curra, \epsilon>0$, $\bonet, \btwot \in [0,1)$
   \begin{algorithmic}[1]
      \State Set $m_0 = 0, v_0 = 0$
      \For{$t = 1$ \textbf{to} $T$}
         \State Draw $s_t \sim \dist$
         \State Compute $\currg = \nabla \fst(\currw)$
         \State $\currm = \bonet \prevm + (1-\bonet) \currg$
         \State $\curre = \frac{1}{\sqrt{\prevv} + \epsilon}$
         \State $\nextw = \currw - \curra \cdot \curre \odot \currm$
         \State $\currv = \btwot \prevv + (1-\btwot) \currg^2$
      \EndFor
   \end{algorithmic}
\end{algorithm}

\begin{algorithm}[t]
   \caption{\textsc{AvaGrad}}
   \label{alg:avagrad}
   \textbf{Input:}
      $\w_1 \in \R^d$, $\curra, \epsilon>0$, $\bonet, \btwot \in [0,1)$
   \begin{algorithmic}[1]
      \State Set $m_0 = 0, v_0 = 0$
      \For{$t = 1$ \textbf{to} $T$}
         \State Draw $s_t \sim \dist$
         \State Compute $\currg = \nabla \fst(\currw)$
         \State $\currm = \bonet \prevm + (1-\bonet) \currg$
         \State $\curre = \frac{1}{\sqrt{\prevv} + \epsilon}$
         \State $\nextw = \currw - \curra \cdot \frac{\curre}{ \norm{\curre / \sqrt d}_2} \odot \currm$
         \State $\currv = \btwot \prevv + (1-\btwot) \currg^2$
      \EndFor
   \end{algorithmic}
\end{algorithm}

Now, we consider the implications of \thmref{thm:dadamconv} for Delayed Adam,
where $\curre = \frac{1}{\sqrt \prevv + \epsilon}$, and hence
$1 / (\gradbtwo + \epsilon) \leq \currei \leq 1 / \epsilon$
for all $t$ and $i$.

For a fixed $\gamma_t = \gamma$, chosen a-priori (that is, without knowledge of
the realization of $\{\curre\}_{t=1}^T$), we can optimize $\gamma$ to minimize
the worst-case rate using $\norm{\curre}^2 \leq d/\epsilon^2$ and
$\min_i \currei \geq \gradbtwo + \epsilon$.  This yields
$\gamma^* = O(\epsilon)$, and a convergence rate linear in $1 / \epsilon$,
suggesting that, at least in the worst case, $\epsilon$ should be chosen to be
as large as possible, and the learning rate $\alpha$ should scale linearly with
$\epsilon$.

What if we allow $\gamma_t$ to vary in each time step?  For example, choosing
$\gamma_t = 1 / \norm{\curre}$ yields a convergence rate of
\begin{equation}
  \sqrt{\frac{\smooth \fdist \gradb^2}{2T}} \cdot
            \expec{S \sim \dist^T}{\left(\frac{1}{2T} \sum_{t=1}^T \frac{\min_i \currei}{\norm{\curre}}\right)^{-1}}.
\label{eq:convrate2}
\end{equation}
Using $1 / (\gradbtwo + \epsilon) \leq \currei \leq 1 / \epsilon$ we see
that in the worst-case this is also linear in $1/ \epsilon$. However, this dependence differs
from the one with fixed $\gamma_t = \gamma$ in a few aspects.  Most notably, if
we consider different scalings of $\curre = \frac{1}{\sqrt {\prevv} + \epsilon}$
(\eg small $\epsilon$ and scaling $\prevv$), the convergence rate with fixed
$\gamma$ can get arbitrarily worse: in \eqref{eq:convrate}, we get that the numerator grows quadratically with the scaling, while the denominator only grows linearly.

On the other hand, for
$\gamma_t = 1 / \norm{\curre}$ the convergence rate remains unchanged since in \eqref{eq:convrate2}, both $\min_i \currei$ and $\norm{\curre}$ grow linearly with the scaling and hence its effect is cancelled out.

For the particular case $d=1$, we have
$\frac{\min_i \currei}{\norm{\curre}} = 1$ in \eqref{eq:convrate2}, yielding the exact same convergence rate as SGD \emph{including constant factors}. However, a constant
$\gamma$ results in a dependence on the scale of $\curre$, which
again can be large if the second moment estimate is either large or small.  Lastly, normalizing
the learning rate by $1 / \norm{\curre}$ removes its dependence on
$\epsilon$ in the worst-case setting, making the two hyperparameters
more separable.

Motivated by the above observations, the choice of $\gamma_t = \sqrt d / \norm{\curre}$, where $\sqrt d$ is added for normalization effects, yields a method which we name AvaGrad --
\textbf{A}daptive \textbf{VA}riance \textbf{Grad}ients, presented as
pseudo-code in Algorithm \ref{alg:avagrad}.  We call it an
``adaptive variance'' method since, if we scale up or down the variance of the
gradients, the convergence guarantee in
\thmref{thm:dadamconv} does not change, while for global learning rates that are independent of $\curre$ (as in Adam and other adaptive methods), it can be arbitrarily bad.

\section{Experiments}
\label{sec:exp}

\begin{figure*}[!bt]
   \centering
   \begin{subfigure}{.5\textwidth}
      \centering
      \includegraphics[width=\linewidth]{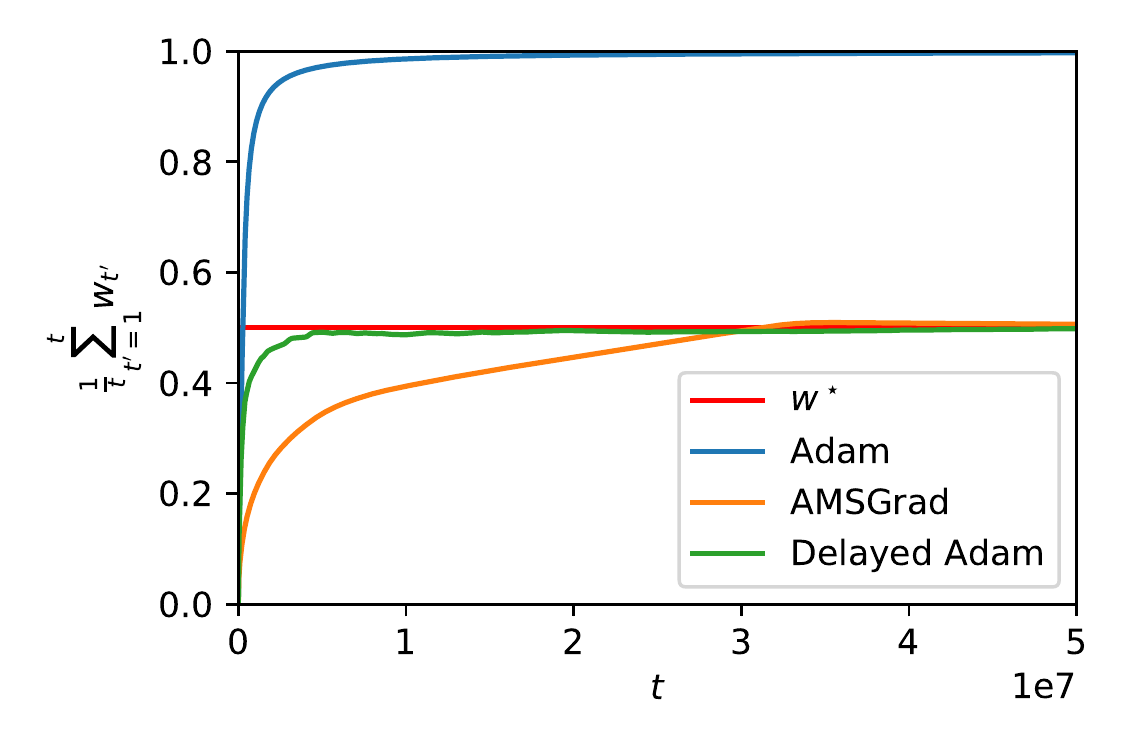}
   \end{subfigure}%
   \hfill
   \begin{subfigure}{.5\textwidth}
      \centering
      \includegraphics[width=\linewidth]{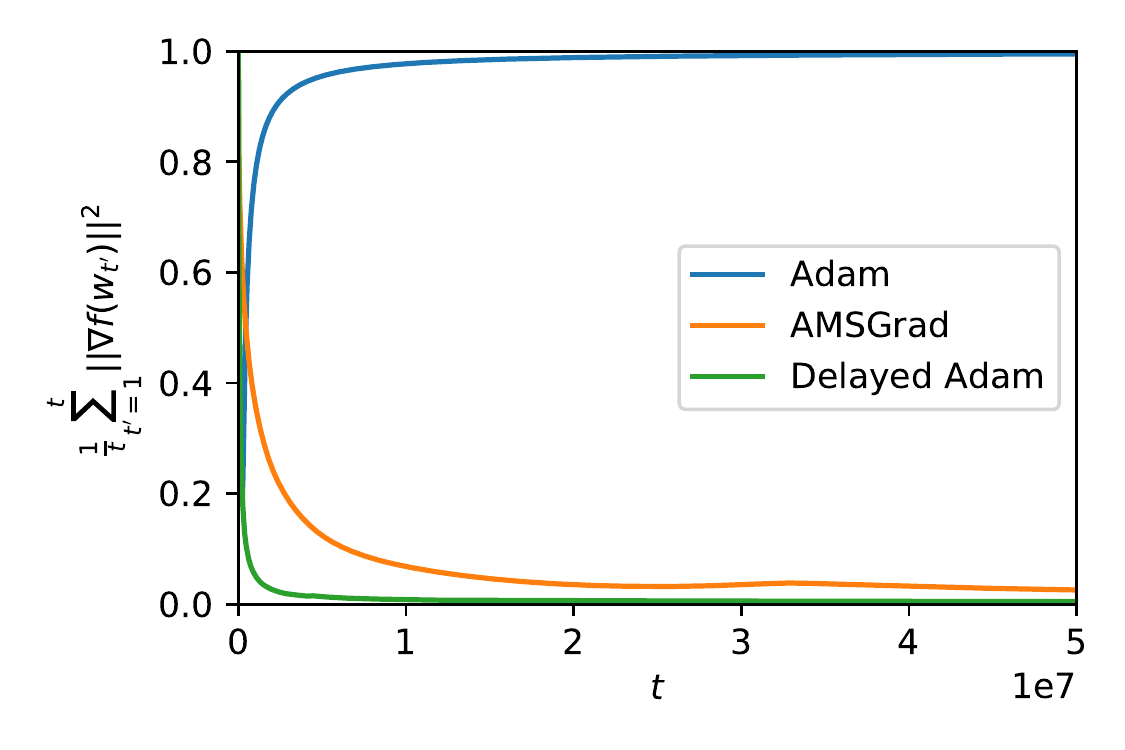}
   \end{subfigure}
   \caption{
      Plots of Adam, AMSGrad, and Delayed Adam trained on the synthetic
      example in Equation ~\ref{eq:synth}, with a stationary point at
      $w^\star \approx 0.5$.
      \textbf{Left:}
         The expected iterate sampled uniformly from $\{\w_1, \dots, \w_t\}$,
         for each iteration $t$. As predicted by our theoretical results, Adam
         moves towards $\w=1$ with $\normed{\nabla f(\w)} = 1$, while Delayed
         Adam converges to $\w^\star$.
      \textbf{Right:}
         The expected norm squared of the gradient, for $\w$ randomly sampled
         from $\{w_1, \dots, \w_t\}$.  Delayed Adam converges significantly
         faster than AMSGrad, while Adam fails to converge.
   }
   \label{fig:synth}
\end{figure*}

\subsection{Synthetic Data}

To illustrate empirically the implications of \thmref{thm:adamdiv} and
\thmref{thm:dadamconv}, we set up a synthetic stochastic optimization problem
with the same form as the one used in the proof of \thmref{thm:adamdiv}:
\begin{equation}
\begin{split}
   & \min_{\w \in [0, 1]} f(\w) \coloneqq \expec{}{\fs(\w)} \\
   & \fs(\w) =
   \begin{cases}
      999 \frac{\w^2}2, \quad \text{with probability } \quad 0.002\\
      -\w, \quad \text{otherwise}
   \end{cases}
   \label{eq:synth}
\end{split}
\end{equation}
This function has a stationary point
$\w^\star = \frac{1-0.002}{999 \cdot 0.002} \approx 0.5$, and it satisfies
\thmref{thm:adamdiv} for $\bone = 0, \btwo = 0.99, \epsilon=10^{-8}$.
We proceed to perform stochastic optimization with Adam, AMSGrad, and Delayed
Adam, with constant learning rate $\curra = 10^{-5}$.  For simplicity, we let
$\cal P$ be uniform over $(1, \dots, T)$, since $\curra$ is constant.

Figure~\ref{fig:synth} shows the progress of $\frac1t \sum_{t'=1}^t \w_{t'}$
and $\frac1t \sum_{t'=1}^t \normed{\nabla f(\w_{t'})}^2$ for each iteration
$t$: as expected, Adam fails to converge to the stationary point $\w^\star$,
while both AMSGrad and Delayed Adam converge.  Note that Delayed Adam converges
significantly faster, likely because it has no constraint on the learning
rates.

\subsection{Image Classification on CIFAR}
\label{sec:cifar}
\begin{figure*}[bt!]
   \centering
   \begin{subfigure}{.5\textwidth}
      \centering
      \scriptsize{\textsf{Adam}}\\
      \includegraphics[width=1.0\linewidth]{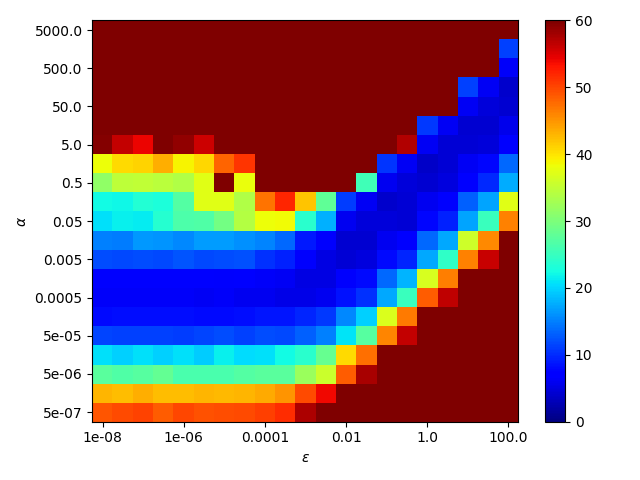}
   \end{subfigure}%
   \hfill
   \begin{subfigure}{.5\textwidth}
      \centering
      \scriptsize{\textsf{AvaGrad}}\\
      \includegraphics[width=1.0\linewidth]{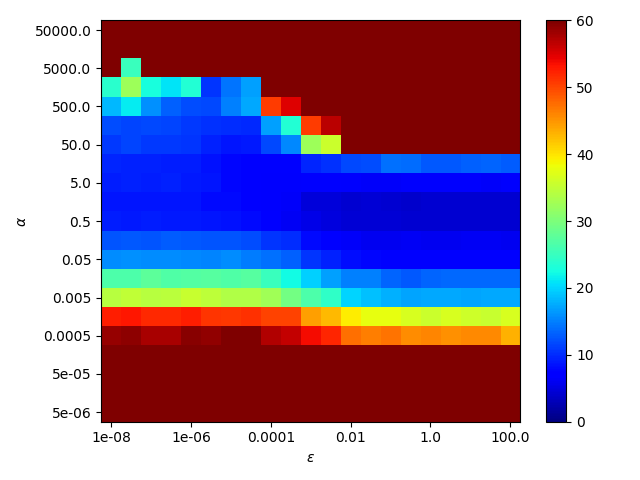}
   \end{subfigure}
   \caption{
      Validation error of a Wide ResNet 28-4 trained on the CIFAR-10
      dataset with Adam (\textbf{left}) and AvaGrad (\textbf{right}), for different values of the learning rate $\alpha$ and
      parameter $\epsilon$, where larger $\epsilon$ yields less
      adaptability. Best performance is achieved with small adaptability ($\epsilon > 0.001$).
   }
   \label{fig:cnns}
\end{figure*}

Our theory suggests that, in the worst case, $\epsilon$ should be chosen as large as possible, at which point the learning rate $\alpha$ should scale linearly with it. As a first experiment to assess this hypothesis, we analyze the interaction between $\alpha$ and $\epsilon$ when training a Wide ResNet 28-4 \citep{wide}
on the CIFAR dataset \citep{cifar}. The CIFAR-10 and CIFAR-100 datasets consist of 60,000 RGB images with $32 \times 32$ pixels and comes with a standard train/test split of 50,000 and 10,000 images, respectively.

Following \citet{wide}, we pre-process the dataset by performing channel-wise normalization using statistics computed from the training set. We also flip each image horizontally with $50\%$ probability and perform random cropping by first padding 4 black pixels to each image and then extracting a random $32 \times 32$ crop.

We use a validation set of 5,000 images to evaluate the performance of SGD and different adaptive gradient methods:
Adam, AMSGrad, AdaBound \citep{adabound, adabound2}, AdaShift \citep{adashift}, and our proposed algorithm, AvaGrad. We also
assess whether performing weight decay as proposed in \citet{adamw} instead
of standard $L_2$ regularization positively impacts the performance of adaptive
methods: we do this by evaluating AdamW and AvaGradW.

The learning rate is decayed by a factor of 5 at epochs 60,
120 and 160, and the model is trained for a total of 200 epochs with a weight
decay of $0.0005$. We use a mini-batch size of 128, and each model is trained on a single GPU. For SGD, we use a momentum of $0.9$, while for adaptive methods we use the default $\bone = 0.9$ and $\btwo = 0.999$. For AdaBound, we use the default final learning rate $\alpha^* = 0.1$ and $\gamma = 10^{-3}$ for the bound functions. Finally, for AdaShift, we use the default stack size $n=10$.

We run each adaptive method with different powers of $10$ for $\epsilon$, multiplied by $1$ and $2$, from $10^{-8}$ up to $100$, a value large enough such that adaptability can be effectively ignored. We also vary the learning rate $\alpha$
of each method with different powers of $10$, multiplied by $1$ and $5$, from $5 \times 10^{-7}$ up to $5000$. In total, we evaluate $441$ different hyperparameter settings for each adaptive method.

Figure \ref{fig:cnns} shows the results for
Adam and AvaGrad.  Our main findings are twofold:

\vspace{-5.0pt}
\begin{itemize}[leftmargin=0.2in]
   \item{
      The optimal $\epsilon$ for every adaptive method is considerably larger
      than the values typically used in practice, ranging from $0.1$ (Adam,
      AMSGrad, AvaGradW) to $10.0$ (AvaGrad, AdamW).  For Adam and AMSGrad, the
      optimal learning rate is $\alpha = \epsilon = 0.1$, a value $100$ times
      larger than the default.
   }
   \item{
      All adaptive methods, except for AdaBound, outperform SGD in terms of
      validation performance.  Note that for SGD the optimal learning rate is
      $\alpha = 0.1$, matching the value used in work such as
      \citet{resnet1, wide, resnext}, which presented state-of-the-art results
      at time of publication.
   }
\end{itemize}

\begin{figure*}[bt!]
   \centering
   \begin{subfigure}{.5\textwidth}
      \centering
      \scriptsize{\textsf{Adam}}\\
      \includegraphics[width=\linewidth]{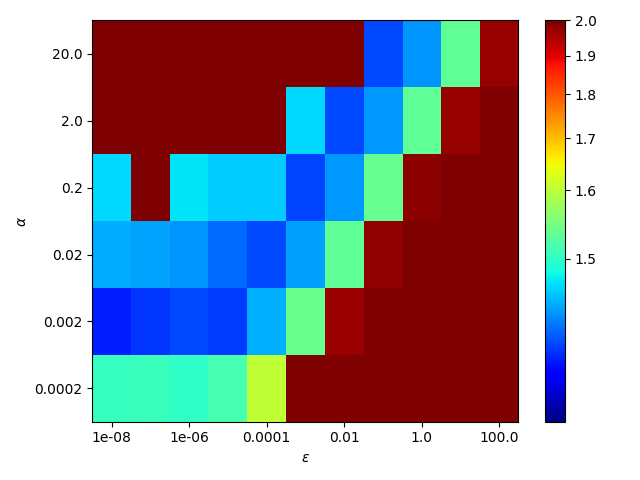}
   \end{subfigure}%
   \hfill
   \begin{subfigure}{.5\textwidth}
      \centering
      \scriptsize{\textsf{AvaGrad}}\\
      \includegraphics[width=\linewidth]{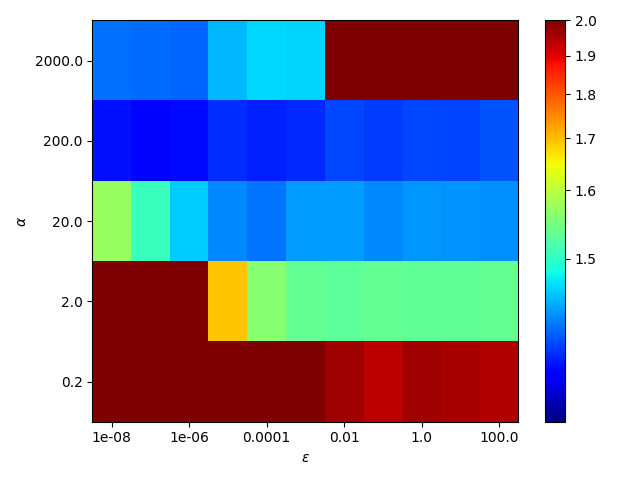}
   \end{subfigure}
   \caption{
      Validation bits-per-character (\emph{lower is better}) of a 3-layer LSTM
      with 300 hidden units, trained on the Penn Treebank dataset with Adam (\textbf{left}) and AvaGrad (\textbf{right}), for
      different values of the learning rate $\alpha$ and
      parameter $\epsilon$, where larger $\epsilon$ yields less adaptability. Best performance is achieved with high adaptability ($\epsilon < 0.0001$).
   }
   \label{fig:rnns}
\end{figure*}

\begin{table*}
   \caption{
      Test performance of SGD and popular adaptive methods in benchmark tasks. Red indicates results with the recommended optimizer, following the paper that proposed each model, and any improved performance is given in blue. The best result for each task is in bold, and numbers in parentheses present standard deviations of 3 runs for CIFAR.
   }
   \label{tab:results}
   \begin{center}
   \begin{tabular}{|c|c|c|c|c|}
%%%%%%%%%%%%%%%%%%%%%%%%%%%%%%%%%%%%%%%%%%%%%%%%%%%%%%%%%%%%%%%%%%%%%%%%%%%%%%%%%%%%%%%%%%%%%%%%
\hline
Method      & \begin{tabular}[x]{@{}c@{}}CIFAR-10\\(Test Err \%)\end{tabular} &
              \begin{tabular}[x]{@{}c@{}}CIFAR-100\\(Test Err \%)\end{tabular} &
              \begin{tabular}[x]{@{}c@{}}ImageNet\\(Top-1 Val Err \%)\end{tabular} &
              \begin{tabular}[x]{@{}c@{}}Penn Treebank\\(Test Bits per Character)\end{tabular}\\
\hline
SGD         & {\color{red} 3.86 (0.08)} & {\color{red} 19.05 (0.24)} & {\color{red} 24.01} & 1.238   \\ \hdashline
Adam        & {\color{blue} \textbf{3.64 (0.06)}} & {\color{blue}18.96 (0.21)} & \textbf{{\color{blue} 23.45}} & {\color{red} 1.182} \\
AMSGrad     & 3.90 (0.17) & {\color{blue} 18.97 (0.09)} & {\color{blue} 23.46} & 1.187 \\
AdaBound    & 5.40 (0.24) & 22.76 (0.17) & 27.99 & 2.863 \\
AdaShift    & 4.08 (0.11) & {\color{blue} 18.88 (0.06)} & N/A & 1.274 \\
AdamW    	& 4.11 (0.17) & 20.13 (0.22) & 27.10 & 1.230 \\  \hline
AvaGrad     & {\color{blue} 3.80 (0.02)} & {\color{blue}\textbf{18.76 (0.20)}} & {\color{blue} 23.58} & {\color{blue} 1.179} \\
AvaGradW     & 3.97 (0.02) & {\color{blue} 19.04 (0.37)} & {\color{blue} 23.49} & {\color{blue} \textbf{1.175}} \\
\hline
%%%%%%%%%%%%%%%%%%%%%%%%%%%%%%%%%%%%%%%%%%%%%%%%%%%%%%%%%%%%%%%%%%%%%%%%%%%%%%%%%%%%%%%%%%%%%%%%
   \end{tabular}
   \end{center}
\end{table*}

However, the fact that adaptive methods outperform SGD in this setting is not
conclusive, since they are executed with more hyperparameter settings (varying
$\epsilon$ as well as $\alpha$).  Moreover, the main motivation for adaptive
methods is to be less sensitive to hyperparameter values; performing an
extensive grid search defeats their purpose.

Aiming for a fair comparison between SGD and adaptive methods, we also train a Wide ResNet 28-10 on both
CIFAR-10 and CIFAR-100, evaluating the test performance of each adaptive method with its
optimal values for $\alpha$ and $\epsilon$ found in the previous experiment. For SGD, we confirmed that the learning  rate $\alpha=0.1$ still yielded the best validation performance with the new architecture, hence the fact that we transfer hyperparameters from the Wide ResNet 28-4 runs does not unfairly advantage adaptive methods in the comparison with SGD. With a larger network and a different task (CIFAR-100), this experiment should
also capture how hyperparameters of adaptive methods transfer between tasks and
models.

On CIFAR-10, SGD achieves $3.86\%$ test error (reported as $4\%$ in
\citet{wide}) and is outperformed by both
      Adam ($3.64\%$) and
   AvaGrad ($3.80\%$).
On CIFAR-100, SGD ($19.05\%$) is outperformed by
    Adam ($18.96\%$),
 AMSGrad ($18.97\%$),
 AdaShift ($18.88\%$),
 AvaGrad ($18.76\%$), and
AvaGradW ($19.04\%$).
We believe these results are surprising, as they show that adaptive methods
can yield state-of-the-art performance when training CNNs as long as their
adaptability is correctly controlled with $\epsilon$.

\subsection{Image Classification on ImageNet}

As a final evaluation of the role of adaptability when training convolutional
networks, we repeat the previous experiment on the ImageNet dataset
\citep{imagenet}, a challenging benchmark composed of 1.2M training and 50,000 validation RGB images sampled from a total of 1,000 classes. We follow \citet{gross} for data augmentation (scale and color transformations) and use 224x224 single-crops to compute the top-1 accuracy on the validation set.

We train a ResNet-50 \citep{resnet2} with SGD and different
adaptive methods, transferring the hyperparameters from our original CIFAR-10
results.  The network is trained for 100 epochs
with a batch size of $256$ on 4 GPUs (batch size of $64$ per GPU), the learning rate is decayed by a factor of 10 at
epochs 30, 60 and 90, and a weight decay of $0.0001$ is applied. 

SGD yields
$24.01\%$ top-1 validation error, underperforming
    Adam ($23.45\%$),
 AMSGrad ($23.46\%$),
 AvaGrad ($23.58\%$) and
AvaGradW ($23.49\%$) -- a total of 4 out of the 6 adaptive methods evaluated on the dataset. Table~\ref{tab:results} summarizes the results.

In contrast to numerous papers that surpassed the state-of-the-art on ImageNet by training networks with SGD \citep{vgg, googlenet, resnet1, resnet2, wide, resnext}, our results show that adaptive methods can yield superior results in terms of generalization performance.

Most strikingly, we observed that Adam outperformed more sophisticated methods such as AMSGrad, AdaBound, and AdamW.  Note that the hyperparameter values we used for SGD match the ones in
\citet{resnet1}, \citet{resnet2} and \citet{gross}: an initial learning rate of
$0.1$ with a momentum of $0.9$.

\subsection{Language Modelling with RNNs}

It is perhaps not very surprising that to perform optimally in the image
classification tasks studied previously, adaptive gradient methods required
large values of $\epsilon$, and hence were \emph{barely} adaptive.

Here, we
consider a task where state-of-the-art results are not typically achieved by SGD, but by
adaptive methods with low values for $\epsilon$: language modelling with
recurrent networks. In particular, we perform character-level language
modelling on the Penn Treebank dataset \citep{ptb,ptbc}, which consists of 5.01M/393k/442k training/validation/test tokens, respectively, and a vocabulary size of 10,000.

Following \citet{awd}, we train 3-layer LSTMs \citep{lstm} with a character embedding size of 200 and varying size for the LSTM cells. The model is trained for a total of 500 epochs,
and the learning rate is decayed by $10$ at epochs 300 and 400.  We use a batch size
of 128, a BPTT length of 150, and weight decay of $1.2 \times 10^{-6}$.

As in \citet{awd}, we apply weight dropout with $p=0.5$ to the LSTM's hidden-to-hidden matrix, variational dropout \citep{variationaldropout} with $p=0.1$ for the input/output layers, $p=0.25$ for the LSTM layers, and $p=0.1$ to the columns of the embedding matrix (embedding dropout).

We first evaluate the validation performance of SGD, Adam, AMSGrad, AdaShift, AdaBound, AdamW,
AvaGrad and AvaGradW with varying learning rate $\alpha$ and adaptability
parameter $\epsilon$, when training a 3-layer LSTM with 300 hidden
units in each layer. We vary the learning rate $\alpha$ in powers of $10$ multiplied by 2: from $0.0002$ up to $20$; for the adaptability parameter $\epsilon$, we vary the values in powers of $10$, multiplied by $1$ and $5$, from $10^{-8}$ up to $100$.

Figure \ref{fig:rnns} shows that, in
this task, smaller values for $\epsilon$ are indeed optimal: Adam, AMSGrad and
AvaGrad performed best with $\epsilon = 10^{-8}$. The optimal learning rates for both Adam and AMSGrad, $\alpha = 0.002$, agree with the value used in \citet{awd}. Both AvaGrad and AvaGradW performed best with $\alpha = 200$: the former with $\epsilon=10^{-8}$, the latter with $\epsilon=10^{-5}$.

Next, we train a larger model: a 3-layer
LSTM with 1000 hidden units per layer (the same model used in \citet{awd},
where it was trained with Adam),
choosing values for $\alpha, \epsilon$ which yielded the best validation performance in
the previous experiment. For SGD, we again confirmed that a learning rate of $20$ performed best on the validation set. 

Table~\ref{tab:results} (right column) reports all
results. In this setting, AvaGrad and AvaGradW outperform Adam, achieving
bit-per-characters of $1.179$ and $1.175$ compared to $1.182$. The poor performance of AdaBound could be caused by convergence issues or due to the default values for its hyperparameters: \citet{adabound2} showed that that the bound functions strongly affect the optimizer's behavior and might require careful tuning.

When combined with the previous results, we see that adaptive methods actually \emph{dominate} SGD across tasks of different domains. In particular, both Adam and AvaGrad outperformed SGD in all 4 considered tasks.

\subsection{Hyperparameter Separability and Domain-Independence}

We observed that, given enough budget for hyperparameter tuning, Adam can actually outperform SGD in tasks such as image classification with CNNs, where adaptive methods have traditionally found little success. But can we decrease the cost of hyperparameter search?

One of the main motivations behind AvaGrad is that it removes the dependence
between the learning rate $\alpha$ and the adaptability parameter $\epsilon$,
at least in the worst-case rate of \thmref{thm:dadamconv}.  Observing the
heatmaps in Figure \ref{fig:cnns} and
\ref{fig:rnns}, we can see that indeed AvaGrad offers great
separability between $\alpha$ and $\epsilon$, unlike Adam. 

In particular, for values larger than $0.0001$, $\epsilon$ has little to no interaction with the
learning rate $\alpha$, as opposed to Adam where the optimal $\alpha$ increases
linearly with $\epsilon$. For language modelling on Penn Treebank, the optimal
learning rate for AvaGrad was $\alpha = 200$ \emph{for every choice of
$\epsilon$}, while for image classification on CIFAR-10, we had $\alpha = 1.0$
for all except two values of $\epsilon$.  This
shows that AvaGrad enables a grid search over $\alpha$ and $\epsilon$ (naively, with
quadratic complexity) to be broken into two line searches over $\alpha$ and
$\epsilon$ separately (linear complexity). In the context of Section \ref{sec:cifar}, this leads to a decrease from $21^2=441$ to $2 \times 21 = 42$ total trials for hyperparameter search, only twice as many as SGD's budget. The cost of optimizing both SGD and Adam \emph{with a fixed $\epsilon$} is the same as fully optimizing AvaGrad. Therefore, unless $\epsilon$ is chosen optimally a-priori, AvaGrad dominates both SGD and Adam given the same budget and coarseness for hyperparameter tuning.

\section{Conclusion}
\label{sec:conclusion}

As neural architectures become more complex, with parameters having highly
heterogeneous roles, parameter-wise learning rates are often necessary for
training.  However, adaptive methods have both theoretical and empirical gaps,
with SGD outperforming them in some tasks and having stronger theoretical convergence guarantees.  In this paper, we close
this gap, by first providing a convergence rate guarantee that matches SGD's, and by showing that, with proper hyperparameter tuning, adaptive methods can dominate in both computer vision and natural language processing tasks. Key to our finding is AvaGrad, our proposed optimizer whose adaptability is decoupled from its learning rate.

Our experimental results show that proper tuning of the learning rate together with the adaptability of the method is necessary to achieve optimal results in different domains, where distinct neural network architectures are used across tasks. By enabling this tuning to be performed in linear time, AvaGrad takes a leap towards efficient domain-agnostic training of general neural architectures.

\bibliographystyle{icml2020}
\ifarxiv
   \bibliography{avagrad_arxiv}
\else
   \bibliography{avagrad}
\fi

\onecolumn
\newpage
\part*{Appendix}
\appendix

%%%%%%%%%%%%%%%%%%%%%%%%%%%%%%%%%%%%%%%
\section{Proof of \thmref{thm:adamdiv}}
\label{sec:proof1}

\begin{proof}
Consider the following stochastic optimization problem:
\begin{equation}
   \min_{\w \in [0,1]} f(\w) \coloneqq \expec{s \sim \dist}{\fs(\w)} \qquad
   \fs(\w) =
   \begin{cases}
      C \frac{\w^2}2,
         \quad \text{with probability }
         \quad p \coloneqq \frac{1+\delta}{C+1} \\
      -\w, \quad \text{otherwise}
   \end{cases}
\end{equation}
where $C > \frac{1-p}p > 1 + \frac{\epsilon}{\w_1 \sqrt{1 - \btwo}}$.  Note
that $\nabla f(\w) = p C \w - (1-p)$, and $f$ is minimized at
$\w^\star = \frac{1-p}{Cp} = \frac{C-\delta}{C(1+\delta)}$.

The proof follows closely from \citet{amsgrad}'s linear example for convergence in
suboptimality.  We assume w.l.o.g. that $\bone = 0$.  Consider:
\begin{equation}
   \Delta_t
      = \nextw - \currw
      = - \eta \frac{\currg}{\sqrt{\currv} + \epsilon}
      = - \eta \frac{\currg}{\sqrt{\btwo \prevv + (1-\btwo) \currg^2} +
                             \epsilon}
\end{equation}
\begin{equation}
\begin{split}
   \frac{\expec{}{\Delta_t}}{\eta}
     &= \frac{\expec{}{\nextw - \currw} }{\eta}
      = - \expec{}{\frac{g_t}{\sqrt{\btwo \prevv + (1-\btwo) g_t^2} + \epsilon}} \\
     &=   p  \expec{}{ \underbrace{\frac{- C \currw}{\sqrt{\btwo \prevv + (1-\btwo) C^2 \currw^2} + \epsilon}}_{T_1}} +
       (1-p) \expec{}{ \underbrace{\frac{1}{\sqrt{\btwo \prevv + (1-\btwo)} + \epsilon}}_{T_2}}
\end{split}
\label{eq:onlinebound}
\end{equation}
where the expectation is over all the randomness in the algorithm up to time
$t$, as all expectations to follow in the proof.  Note that $T_1 = 0$ for
$\currw = 0$. For $\currw > 0$ we bound $T_1$ as follows:
\begin{equation}
\begin{split}
   T_1 \geq \frac{- C \currw}{\sqrt{(1-\btwo) C^2 \currw^2}}
          = \frac{- 1 }{\sqrt{1-\btwo}}
\end{split}
\end{equation}

Hence,
$T_1 \geq \min(0, \frac{- 1 }{\sqrt{1-\btwo}}) = \frac{- 1 }{\sqrt{1-\btwo}}$.

As for $T_2$, we have, from Jensen's inequality:
\begin{equation}
\begin{split}
   \expec{}{T_2} \geq
      \frac{1}{\sqrt{\btwo \expec{}{\prevv} + 1-\btwo} + \epsilon}
\end{split}
\end{equation}

Now, remember that $\prevv = (1-\btwo) \sum_{i=1}^{t-1} \btwo^{t-i-1} g_i^2$,
hence:
\begin{equation}
\begin{split}
   \expec{}{\prevv}
   &= (1-\btwo) \sum_{i=1}^{t-1} \btwo^{t-i-1} \expec{}{g_i^2} \\
   &= (1-\btwo) \sum_{i=1}^{t-1} \btwo^{t-i-1} \left(1 -p + p C^2 \expec{}{\currw^2} \right) \\
   &\leq (1-\btwo) \sum_{i=1}^{t-1} \btwo^{t-i-1} \left(1 -p + p C^2 \right) \\
   &\leq (1-\btwo^{t-1}) \left(1 -p + p C^2 \right) \leq (1+\delta) C^2
\end{split}
\end{equation}
and thus:
\begin{equation}
\begin{split}
   \expec{}{T_2} \geq \frac{1}{\sqrt{\btwo (1+\delta) C + 1-\btwo} + \epsilon}
\end{split}
\end{equation}

Plugging in the bounds for $T_1$ and $T_2$ in Equation \ref{eq:onlinebound}:
\begin{equation}
\begin{split}
   \frac{\expec{}{\Delta_t}}{\eta}
   &\geq
      \frac{1+\delta}{C+1} \frac{- 1 }{\sqrt{1-\btwo}} +
      \left(1-\frac{1+\delta}{C+1}\right)
      \frac{1}{\sqrt{\btwo (1+\delta) C + 1-\btwo} + \epsilon}
\end{split}
\end{equation}

Hence, for large enough $C$, and $C \gg \delta$,
$\w^\star \approx \frac{1}{1+\delta}$ while the above quantity becomes
non-negative, and hence $\expec{}{\currw} \geq \w_1$.  In other words, Adam
will, in expectation, drift away from the stationary point, towards $\w = 1$,
at which point $\normed{\nabla f(1)}^2 = \delta$.  For example, $\delta=1$
implies that
$\lim_{T \to \infty}
   \frac1T \sum_{t=1}^T \expec{}{\normed{\nabla f(\currw)}^2} = 1$. To see that $w=1$ is not a stationary point due to the feasibility constraints, check that $\nabla f(1) = 1 > 0$: that is, the negative gradient points \textit{towards} the feasible region.
\end{proof}

%%%%%%%%%%%%%%%%%%%%%%%%%%%%%%%%%%%%%%%%%
\section{Proof of \thmref{thm:dadamconv}}
\label{sec:proof2}

\begin{proof}

Throughout the proof we use the following notation for clarity: 
\begin{equation}
	\highinft = \max_i \currei \quad\quad \lowinft = \min_i \currei
\end{equation}
We start from the fact that $f$ is $\smooth$-smooth:
\begin{equation}
   f(\nextw) \leq
      f(\currw) +
      \langle \nabla f(\currw), \nextw - \currw \rangle +
      \frac{\smooth}2 \normed{\nextw - \currw}^2
\end{equation}
and use the update $\nextw = \currw - \curra \cdot \curre \odot \currm$:
\begin{equation}
\begin{split}
   f(\nextw) & \leq f(\currw) - \curra \left\langle \nabla f(\currw), \currm \odot \curre \right\rangle + \frac{\curra^2 \smooth}2 \normed{ \currm \odot \curre}^2 \\
   &\leq f(\currw) - \curra \left\langle \nabla f(\currw), \currm \odot \curre \right\rangle + \frac{\curra^2 \smooth \gradb^2 \norm{\curre}^2}2 \\
   &\leq f(\currw) - \curra \bonet \left\langle \nabla f(\currw), \prevm \odot \curre \right\rangle - \curra (1 - \bonet) \left\langle \nabla f(\currw), \currg \odot \curre \right\rangle + \frac{\curra^2 \smooth \gradb^2 \norm{\curre}^2}2 \\
   &\leq f(\currw) + \curra \bonet \normed{\nabla f(\currw)} \cdot \normed{\prevm \odot \curre} - \curra (1 - \bonet) \left\langle \nabla f(\currw), \currg \odot \curre \right\rangle + \frac{\curra^2 \smooth \gradb^2 \norm{\curre}^2}2 \\
   &\leq f(\currw) + \curra \bonet \gradbtwo^2 \highinft - \curra (1 - \bonet) \left\langle \nabla f(\currw), \currg \odot \curre \right\rangle + \frac{\curra^2 \smooth \gradb^2 \norm{\curre}^2}2
\end{split}
\end{equation}
where in the first step we used the fact that
$\normed{ \currm \odot \curre}^2 = \sum_{i=1}^d m_{t,i}^2 \currei^2 \leq \max_j m_{t,j}^2 \sum_{i=1}^d \eta_{t,i}^2 \leq \gradb^2 \normed{\curre}^2$,
in the second we used
$\currm = \bonet \prevm + (1-\bonet) \currg$,
in the third we used Cauchy-Schwarz,
and in the fourth we used $\normed{\nabla f(\currw)} \leq \gradbtwo$, along with
$\normed{ \prevm \odot \curre} = \sqrt{\sum_{i=1}^d m_{t-1,i}^2 \currei^2} \leq \max_j \eta_{t,j} \sqrt{\sum_{i=1}^d m_{t-1,i}^2} \leq \gradbtwo \highinft$.

Now, taking the expectation over $s_t$, and using the fact that $\expec{s_t}{\currg} = \nabla f(\currw)$, and that $\curre$, $\curra$ are both independent of $s_t$:
\begin{equation}
\begin{split}
   \expec{s_t}{f(\nextw)}
   &\leq f(\currw) + \curra \bonet \gradbtwo^2 \highinft - \curra (1 - \bonet) \left\langle \nabla f(\currw), \nabla f(\currw) \odot \curre \right\rangle + \frac{\curra^2 \smooth \gradb^2 \norm{\curre}^2}2 \\
   &\leq f(\currw) + \curra \bonet \gradbtwo^2 \highinft - \curra (1 - \bone) \normed{\nabla f(\currw)}^2 \lowinft + \frac{\curra^2 \smooth \gradb^2 \norm{\curre}^2}2
\end{split}
\end{equation}
where in the second step we used
$\bonet \leq \bone$ and
$\left\langle \nabla f(\currw), \nabla f(\currw) \odot \curre \right\rangle = \sum_{i=1}^d \nabla f(\w)_i^2 \currei \geq \min_j \eta_{t,j} \sum_{i=1}^d \nabla f(\w)_i^2 = \lowinft \normed{\nabla f(\w)}^2$.

Re-arranging, we get:
\begin{equation}
\begin{split}
   \curra \lowinft (1 - \bone) \normed{\nabla f(\currw)}^2 & \leq f(\currw) -  \expec{s_t}{f(\nextw)} + \curra \bonet \gradbtwo^2 \highinft + \frac{\curra^2 \smooth \gradb^2 \norm{\curre}^2}2
\end{split}
\label{eq:proof-inter}
\end{equation}

Now, defining
   $p(t) = \frac{\curra \lowinft}Z$, where $Z = \sum_{t=1}^T \curra \lowinft$, dividing by $Z (1 - \bone)$ and summing over $t$:
\begin{equation}
\begin{split}
   \sum_{t=1}^T p(t) \normed{\nabla f(\currw)}^2 & \leq \frac1{Z(1-\bone)} \sum_{t=1}^T  \left( f(\currw) - \expec{s_t}{f(\nextw)} + \curra \bonet \gradbtwo^2 \highinft + \frac{\curra^2 \smooth \gradb^2 \norm{\curre}^2}2 \right)
\end{split}
\end{equation}

Now, taking the conditional expectation over all samples $S$ given $Z$:
\begin{equation}
\begin{split}
   \expec{S}{\sum_{t=1}^T p(t) \normed{\nabla f(\currw)}^2 \Big| Z}
   &\leq \frac1{Z(1-\bone)} \Big( \sum_{t=1}^T \big( \expec{S}{f(\currw) | Z} - \expec{S}{ \expec{s_t}{f(\nextw)} | Z} \big) \\
   &\qquad + \sum_{t=1}^T \expec{}{\curra \bonet \gradbtwo^2 \highinft + \frac{\curra^2 \smooth \gradb^2 \norm{\curre}^2}2 \Big| Z} \Big) \\
   &\leq \frac1{Z(1-\bone)} \Big( \sum_{t=1}^T \big( \expec{S}{f(\currw) | Z} - \expec{S}{f(\nextw) | Z} \big) \\
   &\qquad + \sum_{t=1}^T \expec{}{\curra \bonet \gradbtwo^2 \highinft + \frac{\curra^2 \smooth \gradb^2 \norm{\curre}^2}2 \Big| Z} \Big) \\
   &= \frac1{Z(1-\bone)} \Big(f(w_1) - \expec{S}{f(w_{T+1}) | Z} \\
   &\qquad + \sum_{t=1}^T \expec{}{\curra \bonet \gradbtwo^2 \highinft + \frac{\curra^2 \smooth \gradb^2 \norm{\curre}^2}2 \Big| Z} \Big)
\end{split}
\end{equation}
where in the second step we used $\expec{S}{ \expec{s_t}{f(\nextw)} | Z} = \expec{S}{f(\nextw)}$ which follows from the assumption that $p(Z|s_t) = p(Z)$, and the third step follows from a telescoping sum, along with the fact that $\expec{S}{f(w_1)} = f(w_1)$.  Now, using
$f(w_1) - \expec{S}{f(w_{T+1}) | Z} \leq f(w_1) - f(w^\star) \leq D$:
\begin{equation}
\begin{split}
   \expec{S}{\sum_{t=1}^T p(t) \normed{\nabla f(\currw)}^2 \Big| Z}
   &\leq \frac1{Z(1-\bone)}
      \Big(D +
           \sum_{t=1}^T \expec{S}{\curra \bonet \gradbtwo^2 \highinft +
           \frac{\curra^2 \smooth \gradb^2 \norm{\curre}^2}2 \Big| Z}
      \Big)
\end{split}
\end{equation}

Then, taking the expectation over $Z$:
\begin{equation}
\begin{split}
   \expec{S}{\sum_{t=1}^T p(t) \normed{\nabla f(\currw)}^2}
   &\leq \expec{}{\frac1{Z(1-\bone)} \sum_{t=1}^T \left( \frac{D}{T} + \curra \bonet \gradbtwo^2 \highinft + \frac{\curra^2 \smooth \gradb^2 \norm{\curre}^2}2 \right) }
\end{split}
\end{equation}

Now, let $\curra = \gamma_t \sqrt{\frac{2 \fdist}{T \smooth \gradb^2}}$:
\begin{equation}
\begin{split}
   \expec{S}{\sum_{t=1}^T p(t) \normed{\nabla f(\currw)}^2} &\leq \expec{S}{\frac1{Z (1-\bone)} \sum_{t=1}^T \left(\frac{D}{T} + \gamma_t \bonet \gradbtwo^2 \highinft \sqrt{\frac{2 D}{T \smooth \gradb^2}}  + \frac{D}{T} \gamma_t^2 \norm{\curre}^2 \right)} \\
   &\leq \expec{}{\frac{D}{T \cdot Z (1-\bone)} \sum_{t=1}^T \left(1 + \gamma_t \bonet \highinft \sqrt{\frac{2 d T \gradbtwo^2}{ \smooth D}} + \gamma_t^2 \norm{\curre}^2 \right)}
\end{split}
\end{equation}%

where we used the fact that $\gradbtwo \leq \gradb \sqrt d$.

Now, recall that $Z = \sum_{t=1}^T \curra \lowinft = \sqrt{\frac{2 \fdist}{T \smooth \gradb^2}} \sum_{t=1}^T \gamma_t \lowinft$:
\begin{equation}
\begin{split}
   \expec{S}{\sum_{t=1}^T p(t) \normed{\nabla f(\currw)}^2} &\leq \expec{S}{\frac{1}{(1-\bone)} \sqrt\frac{\smooth \fdist \gradb^2}{2T} \cdot \frac{\sum_{t=1}^T \left(1 + \gamma_t \bonet \highinft \sqrt{\frac{2 d T \gradbtwo^2}{ \smooth D}} + \gamma_t^2 \norm{\curre}^2 \right)}{\sum_{t=1} \gamma_t \lowinft}} \\
   &= \frac{1}{(1-\bone)} \sqrt\frac{\smooth \fdist \gradb^2}{2T} \cdot \expec{S}{\frac{\sum_{t=1}^T \left(1 + \gamma_t \bonet \highinft \sqrt{\frac{2 d T \gradbtwo^2}{ \smooth D}} + \gamma_t^2 \norm{\curre}^2 \right)}{\sum_{t=1} \gamma_t \lowinft}}
\end{split}
\label{eq-proof2inter1}
\end{equation}

Setting $\bonet = \bone = 0$ and checking that $\sum_{t=1}^T p(t) \normed{\nabla f(\currw)}^2 = \expec{t \sim \mathcal P (t|S)}{\normed{\nabla f(\currw)}^2}$:
\begin{equation}
\begin{split}
     \expec{\algrand}{\normed{\nabla f(\currw)}^2} & \leq \sqrt\frac{\smooth \fdist \gradb^2}{2T} \cdot \expec{S \sim \dist^T}{\frac{\sum_{t=1}^T 1 + \gamma_t^2 \norm{\curre}^2 }{\sum_{t=1} \gamma_t \lowinft}}
\end{split}
\label{eq-proof2final}
\end{equation}

Recalling that $\lowinft = \min_i \currei$ proves the claim.

\end{proof}

\subsection{The case with first-order momentum}
\label{sec:proof2-momentum}
For the case $\bonet > 0$, assume that $\bonet = \frac{\bone}{\sqrt t}$ in Equation \ref{eq-proof2inter1}:
\begin{equation}
\begin{split}
     \expec{S}{\sum_{t=1}^T p(t) \normed{\nabla f(\currw)}^2} & \leq \frac{1}{(1-\bone)} \sqrt\frac{\smooth \fdist \gradb^2}{2T} \cdot \expec{S}{\frac{\sum_{t=1}^T \left(1 + \gamma_t \highinft \frac{\bone}{\sqrt t} \sqrt{\frac{2 d T \gradbtwo^2}{ \smooth D}} + \gamma_t^2 \norm{\curre}^2 \right)}{\sum_{t=1} \gamma_t \lowinft}} \\
     & \leq \frac{1}{(1-\bone)} \sqrt\frac{\smooth \fdist \gradb^2}{2T} \cdot \expec{S}{\frac{\bone \sqrt{\frac{2 d T \gradbtwo^2}{ \smooth D}} \left(\max_t \gamma_t \highinft \right) \sum_{t=1}^T \frac{1}{\sqrt t} + \sum_{t=1}^T \left(1 + \gamma_t^2 \norm{\curre}^2 \right)}{\sum_{t=1} \gamma_t \lowinft}} \\   
     & \leq \frac{1}{(1-\bone)} \sqrt\frac{\smooth \fdist \gradb^2}{2T} \cdot \expec{S}{\frac{2 T \bone \sqrt{\frac{2 d \gradbtwo^2}{ \smooth D}} \left(\max_t \gamma_t \highinft \right) + \sum_{t=1}^T \left(1 + \gamma_t^2 \norm{\curre}^2 \right)}{\sum_{t=1} \gamma_t \lowinft}}
\end{split}
\end{equation}

where in the last step we used $\sum_{t=1}^T \frac1{\sqrt t} \leq 2 \sqrt T$.

Similarly to the guarantee in Equation~\ref{eq-proof2final}, we can show a $O(1 / \sqrt T)$ convergence rate if we further assume that there exist constants $\highinf$ and $\lowinf$ such that
$0 < \lowinf \leq \currei \leq \highinf < \infty$ for all $i$ and $t$ (\ie the parameter-wise learning rates are bounded away from zero and also from above), and that $\gamma_t$ is bounded similarly. For example, having $\gamma_t = \gamma$ yields:
\begin{equation}
\begin{split}
     \expec{S}{\sum_{t=1}^T p(t) \normed{\nabla f(\currw)}^2} & \leq \frac{1}{(1-\bone)} \sqrt\frac{\smooth \fdist \gradb^2}{2T} \cdot \expec{S}{\frac{2 T \highinf \gamma \bone  \sqrt{\frac{2 d \gradbtwo^2}{ \smooth D}} + \sum_{t=1}^T \left(1 + \gamma^2 d \highinf^2 \right)}{\sum_{t=1} \gamma \lowinf}} \\
     & = \frac{1}{\lowinf (1-\bone)} \sqrt\frac{\smooth \fdist \gradb^2}{2T} \cdot \left(\gamma^{-1} + 2 \highinf \bone  \sqrt{\frac{2 d \gradbtwo^2}{ \smooth D}} + \gamma d \highinf^2 \right)
\end{split}
\end{equation}

where in the first step we used $\norm{\curre}^2 \leq d \highinf^2$, $\highinft \leq \highinf$, and $\lowinft \geq \lowinf$.

Note that for any constant $\gamma$ the above is $O(1 / \sqrt T)$.

\subsection{The case with unconditional distribution over iterates}
\label{sec:proof2-unconditional}

To show a similar bound without the assumption that $p(Z|s_t) = p(Z)$, we can alternatively bound $\lowinft$ and $\norm{\curre}$ using the worst case over possible samples $S$. From \eqref{eq:proof-inter} we have, with $\bonet = 0$:
\begin{equation}
\begin{split}
	\left( \inf_{S} \lowinft \right) \curra (1 - \bone) \normed{\nabla f(\currw)}^2 & \leq f(\currw) - \expec{s_t}{f(\nextw)} + \frac{\curra^2 L \gradb^2 \left(\sup_S \norm{\curre}^2 \right)}2
\end{split}
\end{equation}
Now, define $p(t) = \alpha_t \left(\inf_S \lowinft\right) / Z$ with $Z = \sum_{t=1}^T \alpha_t \left(\inf_S \lowinft\right)$ instead. As long as $\alpha_t$ does not depend on $S$, $Z$ is no longer a random variable. Following the same steps as above leads to the following:
\begin{equation}
\begin{split}
     \expec{\substack{S \sim \dist^T \\ t \sim \mathcal P(t)}}{\normed{\nabla f(\currw)}^2} & \leq \sqrt\frac{\smooth \fdist \gradb^2}{2T} \cdot \frac{\sum_{t=1}^T 1 + \gamma_t^2 \left(\sup_S \norm{\curre}^2 \right) }{\sum_{t=1} \gamma_t \left(\inf_S \lowinft\right)}
\end{split}
\end{equation}

In particular, if there are constants $\highinf$ and $\lowinf$ such that
$0 < \lowinf \leq \currei \leq \highinf < \infty$ for all $i$ and $t$, can bound $\sup_S \norm{\curre}^2 \leq d \highinf^2$ and $\inf_S \lowinft \geq \lowinf$, yielding:
\begin{equation}
\begin{split}
     \expec{\substack{S \sim \dist^T \\ t \sim \mathcal P(t)}}{\normed{\nabla f(\currw)}^2} & \leq \sqrt\frac{\smooth \fdist \gradb^2}{2T} \cdot \frac{\sum_{t=1}^T 1 + \gamma_t^2 d \highinf^2}{\lowinf \sum_{t=1} \gamma_t}
\end{split}
\end{equation}

Hence a $O(1 / \sqrt T)$ follows as long as $\gamma_t$ can be upper and lower bounded accordingly by constants.

\end{document}